\documentclass[11pt]{article}

\usepackage[margin=1in]{geometry}
\usepackage{amsmath,amssymb,amsthm,mathtools}
\usepackage{booktabs}
\usepackage{tabularx}
\usepackage{tikz}
\usetikzlibrary{positioning}
\usepackage{algorithm}
\usepackage{algpseudocode}
\usepackage{microtype}
\usepackage[numbers,sort&compress]{natbib}
\usepackage[colorlinks=true,allcolors=blue]{hyperref}
\DeclareMathSizes{10.95}{10}{8}{8}

\newtheorem{theorem}{Theorem}
\newtheorem{lemma}[theorem]{Lemma}

\newtheorem{corollary}[theorem]{Corollary}
\newtheorem{remark}[theorem]{Remark}

\newcommand{\E}{\mathbb{E}}
\newcommand{\R}{\mathbb{R}}
\newcommand{\cB}{\mathcal{B}}
\newcommand{\cK}{\mathcal{K}}
\newcommand{\cS}{\mathcal{S}}
\newcommand{\Reg}{\operatorname{Reg}}
\newcommand{\dtot}{d_{\mathrm{tot}}}
\newcommand{\smax}{\sigma_{\max}}
\newcolumntype{Y}{>{\raggedright\arraybackslash}X}

\title{Beyond Peak Backlog: Conditional Energy and Temporal Geometry in\\
Capacity-Constrained Delayed Bandit Optimization}
\author{
Anling Xiang\textsuperscript{1},
Yuwen Yang\textsuperscript{2},
Yang Shen\textsuperscript{3,4}\\
\textsuperscript{1}Department of Intelligent Communication, School of Journalism and Communication,\\
Minzu University of China, Beijing, China\\
\textsuperscript{2}ZeeLin (Beijing) Technology Co., Ltd., Beijing, China\\
\textsuperscript{3}School of Journalism and Communication, Tsinghua University, Beijing, China\\
\textsuperscript{4}College of AI, Tsinghua University, Beijing, China\\
\small ORCID: Anling Xiang \href{https://orcid.org/0000-0003-1690-1586}{0000-0003-1690-1586}\\
\small Yuwen Yang \href{https://orcid.org/0009-0002-3084-6720}{0009-0002-3084-6720};
Yang Shen \href{https://orcid.org/0000-0003-4814-9018}{0000-0003-4814-9018}
}
\date{}

\begin{document}
\maketitle

\begin{abstract}
What is the right delay complexity when a learner can track only $C$ pending
feedback items and discarded feedback is permanently lost?  The current
one-point bandit convex optimization guarantee in this model pays
$\sqrt{T\smax}$, where $\smax$ is the peak backlog, although unlimited
tracking admits the sharper $\sqrt{\dtot}$ dependence on total delay.  This
gap cannot be closed by a direct moment substitution: randomized admission
creates dependent importance weights, and a learning rate adapted to their
history can invalidate the required cancellation.

We introduce a scheduler-side conditional-energy interface that separates
rate adaptation from the one-point perturbation filtration.  Under the same
semi-clairvoyant oracle and pathwise hard-capacity contract, it yields an
untuned learner satisfying
\[
 \begin{gathered}
 O\!\left(GD\left[\sqrt{E_C\dtot}
 +T^{3/4}\left(1+\frac{\smax}{C}\right)^{1/4}
 \sqrt{\nu k}+1\right]\right),\\
 E_C=1+\left\lceil\log_2\!\left(1+\frac{16(\smax+1)}C\right)\right\rceil.
 \end{gathered}
\]
Here $\nu=M/(Gr)$ is the one-point geometry ratio.  Thus peak-times-horizon
is replaced by total backlog area, up to the explicit restart factor $E_C$;
a public constant-factor peak bound removes $E_C$ while $\dtot$ remains
unknown.  The interface is stated as a reusable delayed weighted
Follow-the-Bandit-Leader theorem, and we show why unrestricted predictable
adaptation is false.

Under strong convexity, the interface yields the temporal cost
$H_A(d)=\sum_t\sigma_t/(A+t)$.  Two delay vectors with identical delay
multisets, $\dtot$, $\smax$, and capacity have polynomially different minimax
regret.  Hence total delay controls the convex upper, but timing matters under
curvature.  Finally, a continuous hard family converts tracking capacity into
a zeroth-order query budget and gives the endpoint
\[
 \Omega\!\left(\min\{T,k\sqrt{T\max\{1,d/C\}}\}\right).
\]
This shows that the hard cap is statistically non-vacuous.  The results
require $C\ge\ln T+1$ for the upper bounds and do not constitute a complete
capacity minimax characterization.
\end{abstract}

\section{Introduction}

Delayed online optimization usually assumes that every pending gradient or
function value can eventually be recovered.  In systems with a finite
feedback buffer---for example, outstanding experiments, remote evaluations,
or asynchronous jobs---this assumption fails: only $C$ pending rounds can be
tracked, and preempted feedback is lost forever.  The hard-capacity model of
\citet{rar-capacity-oco-2026} captures exactly this constraint through
$|\cS_t|\le C$ on every sample path.  Its one-point convex BCO guarantee is
\[
 O\!\left(GD\left[\sqrt{T\smax}
 +T^{3/4}\left(1+\frac{\smax}{C}\right)^{1/4}
 \sqrt{\nu k}\right]\right),
\]
and explicitly leaves open whether the first term can be reduced to the
unconstrained-optimal $\sqrt{\dtot}$ dependence.  Since
$\dtot=\sum_t\sigma_t\le T\smax$, the difference can be polynomial when a
large backlog is short-lived.

At first sight, one might keep the time-varying backlog inside the existing
proof and sum it at the end.  This is not valid.  Randomized admission creates
importance weights coupled through the hard tracking set, while the delayed
one-point estimator contains a square root of the pending weighted energy.
Unconditional moment bounds suffice for deterministic rates but cannot be
multiplied by a rate adapted to the same weight history.  We give an explicit
two-round counterexample to this tempting step.

Our solution is to separate \emph{what the scheduler knows} from what the
bandit estimator randomizes.  We define a scheduler-side filtration, derive
an exact conditional energy envelope inside it, and prove an adaptive
delayed-weighted FTBL theorem.  This interface is the paper's central object:
the total-delay and strongly-convex results are two consequences of the same
filtration-safe principle rather than unrelated refinements of the source
analysis; Figure~\ref{fig:dependency-map} records the boundary precisely.

The resulting message has two parts.  For general convex losses, total
backlog area can replace a peak-times-horizon delay charge even when feedback
is permanently censored.  Under curvature, however, the temporal placement
of that same backlog can change minimax regret polynomially.

For compactness, let
$E_C=1+\lceil\log_2(1+16(\smax+1)/C)\rceil$ denote the capacity-relative
number of backlog scales.

\paragraph{Contributions.}
\begin{enumerate}
  \item \emph{A conditional-energy theorem.}  We identify the scheduler-side
  filtration needed by our construction, prove a causal adaptive FTBL
  interface for conditional energy envelopes, and show by counterexample why
  adaptation to the full predictable history is invalid.
  \item \emph{From peak backlog to total delay.}  An exact Bernoulli-proxy
  envelope and delay-clock restart scheme give a hard-capacity learner that
  knows neither $\dtot$ nor $\smax$ and replaces $\sqrt{T\smax}$ by
  $O(\sqrt{E_C\dtot})$.  With a public constant-factor peak bound, the
  restart factor disappears while $\dtot$ remains unknown.
  \item \emph{Temporal geometry under curvature.}  The same interface yields
  the harmonic cost $H_A(d)=\sum_t\sigma_t/(A+t)$.  Two instances with the
  same delay multiset, $\dtot$, $\smax$, and capacity nevertheless have
  polynomially different minimax regret, proving that timing survives after
  all standard aggregate summaries are fixed.
\end{enumerate}
A complementary endpoint turns pathwise tracking capacity into a continuous
zeroth-order query budget.  It verifies that the hard cap is statistically
non-vacuous, but is not presented as a pointwise match to the upper bound.

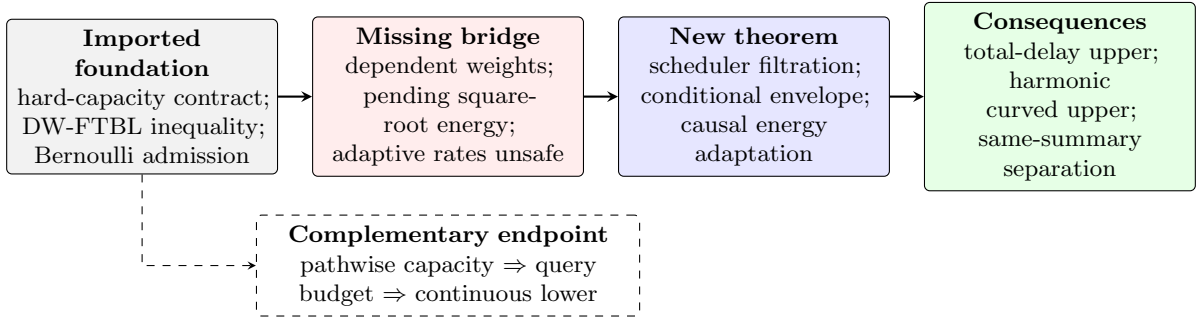
\begin{figure}[t]
\centering
\begin{tikzpicture}[
  >=stealth,
  node distance=0.45cm and 0.45cm,
  box/.style={draw,rounded corners=2pt,align=center,inner sep=4pt,
              text width=0.20\textwidth,font=\footnotesize},
  endpoint/.style={draw,dashed,rounded corners=2pt,align=center,inner sep=4pt,
                   text width=0.29\textwidth,font=\footnotesize}
]
\node[box,fill=gray!10] (source) {\textbf{Imported foundation}\\
  hard-capacity contract;\\DW-FTBL inequality;\\Bernoulli admission};
\node[box,fill=red!7,right=of source] (obstacle) {\textbf{Missing bridge}\\
  dependent weights;\\pending square-root energy;\\adaptive rates unsafe};
\node[box,fill=blue!10,right=of obstacle] (interface) {\textbf{New theorem}\\
  scheduler filtration;\\conditional envelope;\\causal energy adaptation};
\node[box,fill=green!10,right=of interface] (consequences) {\textbf{Consequences}\\
  total-delay upper;\\harmonic curved upper;\\same-summary separation};
\draw[->,thick] (source)--(obstacle);
\draw[->,thick] (obstacle)--(interface);
\draw[->,thick] (interface)--(consequences);
\node[endpoint,below=0.5cm of obstacle] (lower) {\textbf{Complementary endpoint}\\
  pathwise capacity $\Rightarrow$ query budget $\Rightarrow$ continuous lower};
\draw[->,dashed] (source.south) |- (lower.west);
\end{tikzpicture}
\caption{Dependency map.  The contract, base DW-FTBL inequalities, and
Bernoulli admission rule are imported from \citet{rar-capacity-oco-2026}.
The missing step is not a sharper summation: weight-dependent adaptation is
invalid without the blue filtration-safe interface.  The dashed capacity
lower is a supporting endpoint rather than an application of the compiler.}
\label{fig:dependency-map}
\end{figure}

\begin{table}[t]
\centering
\caption{Result map.  Constants depending on fixed geometric scales are
suppressed.  ``Match'' refers only to the stated component, not to a complete
BCO minimax characterization.}
\footnotesize
\begin{tabularx}{\textwidth}{@{}>{\raggedright\arraybackslash}p{0.13\textwidth}Y Y >{\raggedright\arraybackslash}p{0.18\textwidth}@{}}
\toprule
Setting & Closest benchmark & This work & Role / limitation \\
\midrule
Convex BCO
& $\sqrt{T\smax}+T^{3/4}(1+\smax/C)^{1/4}\sqrt{\nu k}$
& $O(\sqrt{E_C\dtot}+T^{3/4}(1+\smax/C)^{1/4}\sqrt{\nu k})$
& Upper; $\sqrt{\dtot}$ endpoint matched up to $\sqrt{E_C}$ \\
\addlinespace
Strongly convex BCO
& $\smax\log T+(T^2\log T)^{1/3}(1+\smax/C)^{1/3}(\nu k)^{2/3}$
& $H_A(d)+A^{1/3}T^{2/3}$
& Upper; public backlog bound \\
\addlinespace
Temporal geometry
& Aggregate summaries assign the same scale to the early/late pair
& $R_T^*(d^{\rm early})=\Omega(T^\beta)$ versus
  $R_T^*(d^{\rm late})=O(T^{2/3})$
& Minimax separation; $2/3<\beta<5/6$ \\
\addlinespace
Fixed-delay capacity
& No continuous-domain capacity lower in the exact contract
& $\Omega(\min\{T,k\sqrt{T\max\{1,d/C\}}\})$
& Lower endpoint; not a full match \\
\bottomrule
\end{tabularx}
\end{table}

\begin{figure}[t]
\centering
\begin{tikzpicture}[x=1cm,y=1cm,>=stealth,font=\small]
  % Early panel.
  \node[font=\bfseries] at (2.5,3.65) {Early placement};
  \draw[->] (0,0)--(5,0) node[right] {$t$};
  \draw[->] (0,0)--(0,3.15);
  \node[rotate=90] at (-0.48,1.55) {backlog $\sigma_t$};
  \fill[blue!16] (0.3,0)--(1.55,2.7)--(2.8,0)--cycle;
  \draw[blue!75!black,very thick] (0.3,0)--(1.55,2.7)--(2.8,0);
  \draw[red!75!black,dashed,thick] (0,1.35)--(4.75,1.35);
  \node[anchor=west,red!75!black] at (3.85,1.55) {$C=h/2$};
  \node[above=2pt] at (1.55,2.7) {$\smax=h$};
  \node at (2.5,-0.48) {$H_A(d)=\Theta(h)$};
  \node[font=\footnotesize] at (2.5,-0.92)
        {$R_T^*=\Omega(T^\beta)$, $\beta>2/3$};

  % Late panel.
  \begin{scope}[xshift=8cm]
    \node[font=\bfseries] at (2.5,3.65) {Late placement};
    \draw[->] (0,0)--(5,0) node[right] {$t$};
    \draw[->] (0,0)--(0,3.15);
    \node[rotate=90] at (-0.48,1.55) {backlog $\sigma_t$};
    \fill[blue!16] (2.2,0)--(3.45,2.7)--(4.7,0)--cycle;
    \draw[blue!75!black,very thick] (2.2,0)--(3.45,2.7)--(4.7,0);
    \draw[red!75!black,dashed,thick] (0,1.35)--(4.75,1.35);
    \node[anchor=west,red!75!black] at (0.12,1.55) {$C=h/2$};
    \node[above=2pt] at (3.45,2.7) {$\smax=h$};
    \node at (2.5,-0.48) {$H_A(d)=\Theta(h^2/T)$};
    \node[font=\footnotesize] at (2.5,-0.92)
          {$R_T^*=O(T^{2/3})$};
  \end{scope}

  \node[align=center,font=\footnotesize] at (6.5,-1.62)
       {Same delay multiset, $\dtot=h^2$, $\smax=h$, and $C=h/2$.\\
        Only the temporal placement changes.};
\end{tikzpicture}
\caption{Schematic continuous envelopes of the discrete early/late backlog
profiles used in Theorem~\ref{thm:separation}.  Peak, area, delay multiset,
and capacity agree, while the harmonic placement cost and minimax regret
differ polynomially.}
\label{fig:backlog-geometry}
\end{figure}
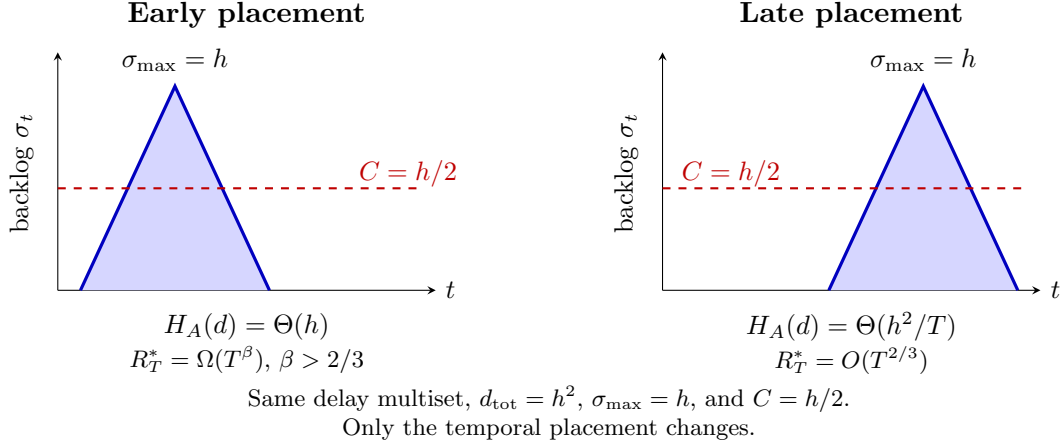

\section{Model and notation}

Let $\cK\subset\R^k$ be convex with diameter at most $D$ and
$r\mathbb B^k\subseteq\cK\subseteq R\mathbb B^k$.  Before play, an
oblivious adversary fixes differentiable convex losses
$f_t:\cK\to\R$ and delays $d_t\in\{0,\ldots,T-t\}$.  Each loss is
$G$-Lipschitz and satisfies $|f_t(x)|\le M$.

At round $t$, the learner queries one point $x_t$.  Its scalar feedback is
observed at time $t+d_t$ if and only if index $t$ has remained in a tracking
set of pathwise size at most $C$.  A preempted index cannot be reinstated.  At
every expiry, including expiry of an untracked index, the learner receives the
indexed message $(t,\text{feedback})$ or $(t,\perp)$.  This is the
semi-clairvoyant model of \citet{rar-capacity-oco-2026}.

Define
\[
 \cB_t=\{s<t:t\le s+d_s\},\qquad
 \sigma_t=|\cB_t|,
\]
\[
 \dtot=\sum_{t=1}^T\sigma_t=\sum_{t=1}^T d_t,\qquad
 \smax=\max_t\sigma_t,\qquad \nu=\frac{M}{Gr}.
\]
The comparator is $x^*\in\arg\min_{x\in\cK}\sum_t f_t(x)$, and all
guarantees are in expectation over learner randomization.
For a nonnegative analysis-weight sequence $w_{1:T}$, write
\[
 \Reg_T^w(u)=\sum_{t=1}^T w_t\bigl(f_t(x_t)-f_t(u)\bigr).
\]
The weights are analysis variables produced by the admission wrapper; a
round with weight zero can still incur true regret and is accounted for by
the wrapper or by the explicit censoring charge below.

\section{Scheduler-side conditional energy}

Let $\mathcal F_t^{\mathrm{sch}}$ contain only the scheduler-side projection
of indexed expiry messages (index, timing, and delivery/empty status, but not
the scalar feedback payload), together with the
tracking set, past proxy coins, admission decisions, importance weights,
public parameters, and scheduler randomness available before the current
coin.  It excludes every past and current one-point perturbation direction and
every function value generated from those directions.  Rates and energy
envelopes below are measurable with respect to this scheduler filtration.
Conditional on the fixed oblivious instance, the entire scheduler and weight
process is generated only from indexed expiry messages, proxy/admission
randomness, and fresh scheduler randomness independent of the full direction
sequence.  No present or future weight may depend on a past direction.

For nonincreasing $\delta_t\in(0,r]$ and nonnegative weights $w_t$, define
\begin{align*}
 \Xi_t={}&\frac{k^2M^2w_t^2}{G^2\delta_t^2}
 +\frac{kMw_t}{G\delta_t}
       \sqrt{\sum_{s\in\cB_t}w_s^2}
 +w_t\sum_{s\in\cB_t}w_s.
\end{align*}

\begin{theorem}[Conditional energy]\label{thm:compiler}
Suppose there are nonnegative
$\mathcal F_t^{\mathrm{sch}}$-measurable $h_t,b_t$ such that
\[
 \E[\Xi_t\mid\mathcal F_t^{\mathrm{sch}}]\le h_t,
 \qquad
 \E[w_t\mid\mathcal F_t^{\mathrm{sch}}]\le b_t.
\]
Let $A_0>0$ be fixed at the start of the run and dominate all $h_t$ in that
run, and put
\[
 A_t=A_0+\sum_{s\le t}h_s,\qquad
 \eta_t=\frac{D}{G\sqrt{A_t}},\qquad \eta_0=\eta_1.
\]
Then scheduler-adaptive delayed weighted FTBL satisfies, for every
$u\in\cK$,
\[
 \E \Reg_T^w(u)
 \le c_0GD\,\E\sqrt{A_T}
 +\frac{3GD}{r}\E\sum_{t=1}^T b_t\delta_t,
\]
where $c_0=5/2+2\sqrt2$ is valid.  Consequently,
\[
 \E \Reg_T^w(u)
 \le c_0GD\sqrt{\E A_0+\sum_t\E h_t}
 +\frac{3GD}{r}\E\sum_t b_t\delta_t.
\]
\end{theorem}

\begin{lemma}[Delay-clock stopping and zero padding]\label{lem:padding}
Suppose a run ends at a time determined only by the global indexed-expiry
clock and public thresholds.  Conditional on the fixed oblivious delay
vector, the endpoint is deterministic and independent of scheduler coins,
directions, and loss values.  Lemma~\ref{lem:source-ftbl} and
Theorem~\ref{thm:compiler} apply after padding every post-end round by zero
analysis weight and zero loss.
\end{lemma}

\paragraph{Why the filtration matters.}
Unconditional energy moments cannot be multiplied by arbitrary predictable
rates.  If $B\sim\mathrm{Ber}(\epsilon)$ is known before round two,
$\Xi_2=B/\epsilon$, and
$\eta_1=B+\epsilon(1-B)$, then $\E\Xi_2=1$ but
$\E[\eta_1\Xi_2]=1$ while $\E\eta_1\E\Xi_2\le2\epsilon$.  Dependence on
past perturbation directions similarly destroys the square-root cancellation
for pending one-point estimators.

\section{Bernoulli proxy energy}

Consider one run whose local backlog never exceeds a public $H$.  Use
\[
 p=\min\left\{1,\frac{C}{8(H+1)}\right\},\quad
 I_t\sim\mathrm{Ber}(p),\quad
 w_t=Q_tI_t/p,\quad Q_t=\mathbf1\{|\cS_t|<C\}.
\]
Let $c_t=|\cS_t|$ before current admission and
$z=kM/(G\delta)$ for a fixed run-wise radius.

\begin{lemma}[Exact proxy energy]\label{lem:proxy}
Before the current proxy coin,
\[
 \E[\Xi_t\mid\mathcal F_t^{\mathrm{sch}}]
 =\frac{Q_t}{p}[z^2+z\sqrt{c_t}+c_t],
 \qquad
 \E[w_t\mid\mathcal F_t^{\mathrm{sch}}]=Q_t.
\]
If the run has length $L$, local backlog area $D_E$, and local backlog sizes
$\sigma_t^E$, then
\[
 \sum_t\E h_t
 \le \frac{z^2L}{p}+z\sqrt{\frac{LD_E}{p}}+D_E.
\]
Moreover, a valid uniform initial envelope is
\[
 A_0=\frac{2(z^2+\min\{C-1,H\}+1)}{p}
 \le \frac{2z^2}{p}+16(H+1).
\]
\end{lemma}

The key point is pathwise: every pending item with nonzero weight remains in
the actual tracking set, so its weight is exactly $1/p$.  No cross-time
independence of the weights is used.

\begin{corollary}[Known-peak total-delay bound]\label{cor:known-peak}
Suppose $C\ge\ln T+1$ and a public
$\smax\le\bar\sigma\le c_b\smax$ is available for a universal constant
$c_b$.  A
single-run scheduler-side adaptive DW-FTBL learner, using
\[
 p=\min\left\{1,\frac{C}{8(\bar\sigma+1)}\right\},\qquad
 \frac{\delta}{r}=\min\left\{1,c_1\sqrt{\nu k}\,T^{-1/4}p^{-1/4}\right\},
\]
and the learning rates of Theorem~\ref{thm:compiler}, satisfies
\[
 \E R_T(x^*)
 =O\!\left(GD\left[\sqrt{\dtot}
 +T^{3/4}\left(1+\frac{\bar\sigma}{C}\right)^{1/4}
 \sqrt{\nu k}+1\right]\right).
\]
The learner does not know $\dtot$.  If $\bar\sigma$ is a constant-factor
upper bound on $\smax$, the capacity term has the same order as the tuned
source expression.
\end{corollary}

Indeed, apply Lemma~\ref{lem:proxy} to the full horizon.  The compiler and
$\sqrt{\bar\sigma+1}\le O(\sqrt{\dtot}+1)$ give
\[
 O\!\left(GD\left[\sqrt{\dtot}
 +\frac{\nu k}{\delta/r}\sqrt{\frac{T+1}{p}}
 +T\frac{\delta}{r}+1\right]\right).
\]
The displayed radius balances the last two nonconstant terms.  In the clipped
regime the target is already linear and the trivial regret bound applies.

\section{Untuned total-delay algorithm}

\begin{algorithm}[t]
\caption{Untuned total-delay DW-FTBL}
\label{alg:convex}
\begin{algorithmic}[1]
\State Set $B_C=\max\{1,\lfloor C/8\rfloor\}$ and initialize
       $j=0$, $H_j=2^jB_C-1$, an empty tracking set, and fresh base.
\For{$t=1,\ldots,T$}
  \State Process all previously delivered indexed messages and update global backlog.
  \If{global backlog exceeds $H_j$}
    \State Charge and preempt every epoch-$j$ item still pending.
    \State Increase $j$ until $H_j$ covers the backlog; restart the base.
  \EndIf
  \State $p_j\gets\min\{1,C/[8(H_j+1)]\}$.
  \State $\delta_j/r\gets\min\{1,c_1\sqrt{\nu k}T^{-1/4}p_j^{-1/4}\}$.
  \State Update $A_t$ from Lemma~\ref{lem:proxy}; use
         $\eta_t=D/(G\sqrt{A_t})$ for the next center.
  \State Query the DW-FTBL perturbation.
  \State Draw $I_t\sim\mathrm{Ber}(p_j)$ and track $t$ with weight $1/p_j$
         iff $I_t=1$ and capacity is available; otherwise use weight zero.
  \State Forward only observed, noncensored weighted feedback to the active base.
\EndFor
\end{algorithmic}
\end{algorithm}

At a restart, an epoch-born item still pending is charged directly by $GD$,
assigned analysis weight zero, and never forwarded to a later base.  Its
eventual empty acknowledgment updates only the external pending clock.

\begin{lemma}[Delay-clock charged censoring]\label{lem:censor}
Let an epoch boundary depend only on the global delay/expiry clock and public
thresholds, not admission coins, weights, directions, or loss values.  Let
$X$ be its epoch-born pending items.  The true epoch regret is at most the
weighted regret on complement rounds, plus $GD|X|$ and the source saturation
charge.  Equivalently, the base may be analyzed on the same trajectory with
all items in $X$ assigned analysis weight zero.
\end{lemma}

\begin{theorem}[Total-delay hard-capacity BCO]\label{thm:convex}
If $C\ge\ln T+1$, Algorithm~\ref{alg:convex}, which knows neither $\dtot$
nor $\smax$, satisfies
\[
 \E R_T(x^*)
 \le cGD\left[\sqrt{E_C\dtot}
 +T^{3/4}\left(1+\frac{\smax}{C}\right)^{1/4}
 \sqrt{\nu k}+1\right],
\]
where
\[
 E_C=1+\left\lceil\log_2\!\left(1+\frac{16(\smax+1)}{C}\right)\right\rceil.
\]
Thus the only epoch factor is $\sqrt{E_C}$ on the total-delay term; the
one-point and smoothing terms carry no restart logarithm.
\end{theorem}

\begin{remark}[Residual restart factor]
The factor $\sqrt{E_C}$ comes solely from summing the comparator costs
$\sqrt{D_j}$ of independently restarted bases.  Keeping one base across a
drop in $p_j$ would require increasing the optimal smoothing radius, whereas
the source FTBL interface assumes a nonincreasing radius.  We therefore do
not claim that $\sqrt{E_C}$ is minimax necessary; removing it requires a new
nonrestart interface rather than a sharper summation of the present proof.
\end{remark}

\section{Strong convexity and harmonic backlog}

Assume each loss is $\lambda$-strongly convex.  Given a public
$\bar\sigma\ge\smax$, define
\[
 p=\min\left\{1,\frac{C}{8(\bar\sigma+1)}\right\},\qquad
 A=\frac{(\nu k)^2}{p},\qquad
 H_A(d)=\sum_{t=1}^T\frac{\sigma_t}{A+t}.
\]

\begin{algorithm}[t]
\caption{Tuned harmonic-backlog DW-FTBL}
\label{alg:strong}
\begin{algorithmic}[1]
\For{$t=1,\ldots,T$}
  \State $\eta_t\gets1/[\lambda(A+t)]$ and
         $\delta_t\gets r[A/(A+t)]^{1/3}$.
  \State Query the strongly-convex DW-FTBL perturbation.
  \State Draw $I_t\sim\mathrm{Ber}(p)$ and track with weight $1/p$ iff
         $I_t=1$ and capacity is available; otherwise use weight zero.
  \State Forward observed weighted feedback as in the source wrapper.
\EndFor
\end{algorithmic}
\end{algorithm}

\begin{theorem}[Harmonic-backlog strongly-convex BCO]\label{thm:strong}
Suppose $C\ge\ln T+1$ and $\bar\sigma\ge\smax$ is public.  Then
Algorithm~\ref{alg:strong} satisfies
\[
 \E R_T(x^*)
 \le c\frac{G^2}{\lambda}
 \left[1+H_A(d)+A^{1/3}T^{2/3}\right].
\]
In particular,
\[
 H_A(d)\le\smax\log(1+T/A),
\]
and
\[
 A^{1/3}T^{2/3}
 \le cT^{2/3}\left(1+\frac{\bar\sigma}{C}\right)^{1/3}
 (\nu k)^{2/3}.
\]
When $\bar\sigma=\smax$, or is a constant-factor tight upper bound, these
displays strictly refine the corresponding source guarantee.
\end{theorem}

The shifted rates produce regularizer increments
$\alpha_1=\lambda(A+1)$ and $\alpha_t=\lambda$ for $t\ge2$.  The initial
excess is absorbed by the one-point term when $A\le T$; when $A>T$, the
target is already at least linear and the trivial regret bound applies.

\subsection{Same-summary temporal minimax separation}

The harmonic upper suggests that the placement of missing feedback, not only
its area or peak, can survive in the regret.  We now close this implication at
the minimax level for a matched pair of delay vectors, using the same learner
class and public information on both sides.

Let exactly $h$ rounds have delay $h$ and all others delay zero.  In the early
vector, those rounds start at $1,\ldots,h$; in the late vector, they start at
$T-2h+1,\ldots,T-h$.  Both vectors have the same delay multiset,
$\dtot=h^2$, and $\smax=h$, but
\[
 H_A(d^{\mathrm{early}})=\Theta(h),\qquad
 H_A(d^{\mathrm{late}})=\Theta(h^2/T)
\]
for $A\le h$ and $T\ge8h$.

Let $\mathfrak F$ be the class of differentiable functions on $[-1,1]$ that
are $1$-strongly convex, $2$-Lipschitz, and bounded in absolute value by one.
Let $\mathfrak A(h,C)$ be the common class of randomized learners obeying the
semi-clairvoyant one-point hard-capacity contract, knowing
$T,h,C,\mathfrak F$ and the public geometric scales but not the future delay
positions.  Define
\[
 R_T^*(d;h,C)=
 \inf_{\mathcal A\in\mathfrak A(h,C)}
 \sup_{f_{1:T}\in\mathfrak F^T}\E R_T(\mathcal A,f_{1:T},d).
\]
The early lower below remains valid even if the entire delay vector is given
to the learner in advance.  We set $C=\lfloor h/2\rfloor<\smax$, so capacity
is non-vacuous; the early lower itself, however, is caused by delayed information
and remains valid with unlimited capacity.

\begin{theorem}[Same-summary minimax separation]\label{thm:separation}
Fix $\beta\in(2/3,5/6)$ and let $h=\lfloor T^\beta\rfloor$.  For all
sufficiently large $T$, the early and late vectors above, with the common
capacity $C=\lfloor h/2\rfloor$, satisfy
\[
 R_T^*(d^{\mathrm{early}};h,C)\ge ch,
 \qquad
 R_T^*(d^{\mathrm{late}};h,C)
 \le C_0[T^{2/3}+h^2/T+1]=O(T^{2/3}).
\]
Thus the early minimax lower exceeds the late minimax upper by the polynomial
factor $\Omega(T^{\beta-2/3})$.  Both problems use the same public information
and algorithm class.
\end{theorem}

\paragraph{Lower-bound idea.}
Use $\cK=[-1,1]$ and draw a hidden
$\theta\in\{-1,+1\}$ uniformly.  Every loss is
\[
 f_t^\theta(x)=\frac{x^2}{2}-\frac{\theta x}{2}.
\]
The common minimizer is $x^*=\theta/2$ with per-round loss $-1/8$.  Under the
early vector, predictions through round $h+1$ are independent of $\theta$;
hence their expected loss is $\E[x_t^2]/2\ge0$.  Granting $\theta$ after that
cut can only help, and gives expected regret at least $(h+1)/8$.  Averaging
over signs leaves a fixed sign with this regret.  For the late vector,
Theorem~\ref{thm:strong} applies with
$p=C/[8(h+1)]=\Theta(1)$, $A=\Theta(1)$, and
$H_A(d^{\mathrm{late}})=O(h^2/T)$.

\section{Complementary capacity-starvation endpoint}

The preceding separation isolates temporal placement but its early lower is
not caused by capacity.  For completeness, we give a separate endpoint in
which hard capacity directly limits the number of recoverable zeroth-order
observations.  This theorem supports the model's statistical relevance; it is
not a third compiler application or a pointwise match to the convex upper.

\begin{theorem}[Continuous capacity starvation]\label{thm:capacity-lower}
There are universal constants $c,c_0>0$ and a fixed class of stochastic loss
distributions on $\cK=B_2^k$ such that every realized loss is bounded by a
universal constant, $4$-Lipschitz, $1/2$-strongly convex, and $7/2$-smooth.
For every horizon $T$, dimension $k$, capacity $C\ge1$, and integer
$1\le d\le T/8$, every randomized learner obeying the semi-clairvoyant,
preemptive, pathwise hard-capacity one-point contract has an oblivious
deterministic loss sequence and a public delay vector for which
\[
 \E\Reg_T\ge
 c\min\!\left\{T,\;k\sqrt{T\max\{1,d/C\}}\right\}.
 \tag{19}
\]
The first $N=\lfloor T/2\rfloor$ rounds have delay $d$ and all remaining
rounds have delay zero.  Hence $\dtot=\Theta(Td)$ and
$\smax=\Theta(d)$.
\end{theorem}

\paragraph{Proof idea.}
We use the smooth strongly-convex derivative-free family of
\citet{shamir-derivative-free-2013}, replacing Gaussian noise by a compact
mean-zero density whose translated Hellinger distance is quadratic.  Any
procedure with at most $m$ noisy value queries then has optimization error at
least a constant times $\min\{1,k/\sqrt m\}$.  Before the learner generates
$x_N$, every
observed first-block item must have occupied one tracking slot continuously
for $d$ rounds.  Thus pathwise hard capacity gives at most
\[
 m_0=\min\{N,\lfloor CN/d\rfloor\}
\]
informative values.  Simulating only this prefix and outputting its average
reduces the learner to an $m_0$-query
derivative-free procedure.  Convexity yields
\[
 \E\Reg_T\ge cN\min\{1,k/\sqrt{m_0}\},
\]
which is (19).
Appendix~\ref{app:capacity-lower} gives the adaptive Hellinger and
deterministic-sequence details.

\section{Relation to prior work and scope}
\label{sec:related-work}

This section places the paper at the intersection of four established lines:
online convex optimization, learning with delayed feedback, bandit and
zeroth-order convex optimization, and learning with restricted feedback
acquisition.  Table~\ref{tab:literature-map} records which parts of that stack
are imported and where the present arguments enter.

\paragraph{Online convex optimization and adaptive regularization.}
Online gradient descent originates with \citet{zinkevich-oco-2003}; logarithmic
regret under strong convexity was developed by \citet{hazan-log-2007}.
Standard accounts include \citet{cesa-lugosi-book-2006},
\citet{shalev-oco-survey-2012}, and \citet{hazan-oco-book-2016}.
The adaptive-regularization line includes bound optimization
\citep{mcmahan-streeter-2010}, AdaGrad \citep{duchi-adagrad-2011}, predictable
sequences \citep{rakhlin-predictable-2013}, and scale-free FTRL and mirror
descent \citep{orabona-scale-free-2018}.  Our conditional-energy compiler is
in this lineage, but its filtration restriction is essential: the rate may
adapt to scheduler-side information, not to arbitrary past bandit
perturbations or admission coins.

\paragraph{Delayed online learning.}
Early delayed prediction results include \citet{weinberger-delayed-2002} and
the parallel reduction of \citet{zinkevich-slow-2009}.  General reduction
principles were developed by \citet{joulani-delayed-2013}; adversarial delayed
OCO and adaptive-gradient refinements appear in
\citet{quanrud-adversarial-2015} and \citet{joulani-delay-2016}.
For bandits, later work treats cooperation and delay
\citep{cesa-delay-cooperation-2019}, unrestricted or arbitrary delays
\citep{thune-unrestricted-2019,zimmert-arbitrary-2020}, and joint adaptation
to observations and delays \citep{gyorgy-adapting-2021}.  In a complementary
direction, \citet{bistritz-delays-2022} develop no-weighted-regret guarantees
and delay-adaptive doubling arguments for adversarial bandits with delays.

Without hard capacity, delayed one-point BCO has been studied by
\citet{wan-delayed-bco-2024}, while the observation-ordered reduction of
\citet{rar-reduction-2026} reaches the $\sqrt{\dtot}$ scale with unlimited
tracking.  Curvature-aware delayed OCO is treated by
\citet{qiu-curvature-2025}.  Temporal feedback graphs provide a more structural
view: \citet{gatmiry-temporal-2024} encode the complete loss-visibility pattern
as a temporal feedback graph and derive graph-dependent upper and lower regret
bounds.  These works do not combine
one-point convex feedback with permanent censoring induced by a pathwise
tracking cap.

\paragraph{Bandit and zeroth-order convex optimization.}
The oracle-efficient term in this paper descends from the one-point smoothing
estimator of \citet{flaxman-bco-2005}; smooth and strongly-convex variants are
studied by \citet{saha-smooth-bco-2011} and \citet{ito-strong-bco-2020}.
The $T^{3/4}$ term inherited from this route is not the information-theoretic
minimax rate of general BCO.  One-dimensional $\sqrt T$ regret was obtained by
\citet{bubeck-one-dimensional-2015}, and later kernel or information-theoretic
methods attain near-$\sqrt T$ regret with dimension costs
\citep{hazan-li-bco-2016,bubeck-kernel-2017,lattimore-bco-2020}; more recent
work develops computationally efficient Newton-type methods
\citep{fokkema-newton-2024}.  Two-point feedback is substantially stronger
\citep{shamir-two-point-2017}.

Our capacity lower uses the complementary derivative-free query-complexity
line.  Relevant foundations include \citet{jamieson-query-2012}, the one-point
lower bounds of \citet{shamir-derivative-free-2013}, optimal two-evaluation
rates of \citet{duchi-zero-order-2015}, and random gradient-free methods of
\citet{nesterov-random-2017}.  We transport a continuous nonlinear query
barrier through the pathwise occupancy constraint; we do not reduce the lower
bound to a finite collection of independent arms.

\paragraph{Restricted feedback and hard capacity.}
Label-efficient prediction \citep{cesa-label-efficient-2005}, partial
monitoring \citep{cesa-partial-monitoring-2006}, side observations
\citep{mannor-side-observations-2011}, and feedback graphs
\citep{alon-feedback-graphs-2015} all quantify the value of observing only a
subset or transformation of the losses.  They do not impose the present
physical rule that an item must occupy a tracking slot continuously until its
delayed feedback arrives and that a preempted item cannot be restored.

That hard-capacity rule was introduced for finite-action online learning by
\citet{rar-capacity-bandits-2025}.  The exact continuous OCO/BCO contract,
weighted base theorem, and Bernoulli scheduler used here are due to
\citet{rar-capacity-oco-2026}.  Their convex one-point guarantee contains the
delay term $\sqrt{T\smax}$ and explicitly asks whether it can be reduced to the
unconstrained-optimal $\sqrt{\dtot}$ scale or whether a separation is
necessary.  Our known-peak corollary resolves the upper side at that scale;
the untuned theorem incurs only the explicit restart factor $\sqrt{E_C}$.
Figure~\ref{fig:dependency-map} separates this imported stack from the new
filtration-safe interface and its consequences.

\begin{table}[t]
\centering
\caption{Method map for the closest literature.  ``Permanent loss'' means
that feedback from a preempted pending round can never be recovered.}
\label{tab:literature-map}
\footnotesize
\begin{tabularx}{\textwidth}{@{}>{\raggedright\arraybackslash}p{0.20\textwidth}
  >{\centering\arraybackslash}p{0.10\textwidth}
  >{\centering\arraybackslash}p{0.12\textwidth}
  >{\centering\arraybackslash}p{0.12\textwidth}Y@{}}
\toprule
Line of work & one-point convex & hard pathwise cap & permanent loss & principal control quantity \\
\midrule
Classical delayed OCO
  \citep{joulani-delayed-2013,quanrud-adversarial-2015,joulani-delay-2016}
  & no & no & no & delay count or total delay \\
Delayed BCO \citep{wan-delayed-bco-2024,rar-reduction-2026}
  & yes & no & no & $d_{\max}$ or $\dtot$ \\
Temporal feedback graphs \citep{gatmiry-temporal-2024}
  & no & no & no & temporal graph complexity \\
Capacity-constrained finite actions \citep{rar-capacity-bandits-2025}
  & no & yes & yes & capacity--delay throughput \\
Capacity-constrained OCO/BCO \citep{rar-capacity-oco-2026}
  & yes & yes & yes & $\smax$, $C$, weighted backlog \\
This paper, convex result
  & yes & yes & yes & conditional energy and $\dtot$ \\
This paper, curved result
  & yes & yes & yes & $H_A(d)=\sum_t\sigma_t/(A+t)$ \\
\bottomrule
\end{tabularx}
\end{table}

\paragraph{Exact scope of the present results.}
The total-delay theorem is not a new unconstrained delay reduction; it is a
capacity-preserving extension through conditional energy.  The
strongly-convex quantity $H_A$ is likewise not the first time-weighted delay
sum: related strongly-convex analyses control delay through sums weighted by
inverse time or inverse curvature accumulation
\citep{qiu-curvature-2025,rar-reduction-2026}.  Our delta is a shifted primal harmonic bound under
one-point feedback, dependent importance weights, permanent censoring, and a
pathwise cap, together with a same-summary separation in that exact model.

For fixed delay $d$, the standard unlimited-capacity convex lower endpoint is
\citep{rar-capacity-oco-2026}
\[
  \Omega\!\left(GD\min\{T,\sqrt{T(d+1)}\}\right).
\]
Since
$\dtot=\Theta(Td)$ away from the truncated tail and hard capacity cannot help
the learner, this endpoint matches the $\sqrt{\dtot}$ component of our convex
upper.  It does not establish necessity of the capacity-dependent one-point
term.  Our continuous capacity-starvation theorem is a complementary endpoint,
not a pointwise match to the full one-point term
\[
  T^{3/4}\left(1+\frac{\smax}{C}\right)^{1/4}\sqrt{k}.
\]

We do not claim a new feedback model, scheduler, or complete BCO minimax rate.
We do not prove a pointwise matching lower bound for the precise
capacity-dependent one-point factor, a pointwise lower bound in terms of
$H_A(d)$ for every delay vector, the regime $C<\log T$, adaptive-adversary
guarantees, or high-probability regret.  The strongly-convex theorem is tuned
with a public backlog upper bound.

\bibliographystyle{unsrtnat}
\bibliography{references}

\appendix
\section{Imported source interfaces}
\label{app:source-interfaces}

We record the exact portions of \citet{rar-capacity-oco-2026} used in the
paper.  This separates source assumptions from the conditional-energy and
backlog arguments introduced here.

\begin{lemma}[Weighted one-point FTBL interface]
\label{lem:source-ftbl}
Fix an oblivious loss and delay sequence.  Let $w_t\ge0$ be fixed analysis
weights, let $\eta_t>0$ and $\delta_t\in(0,r]$ be nonincreasing, and use the
source one-point DW-FTBL centers and independent uniform perturbation
directions.  Then, for every $u\in\cK$,
\[
 \E_U\Reg_T^w(u)
 \le \frac{D^2}{2\eta_T}
 +2G^2\sum_{t=1}^T\eta_{t-1}\Xi_t
 +\frac{3GD}{r}\sum_{t=1}^T w_t\delta_t.
\]
Here $\eta_0=\eta_1$ and $\Xi_t$ is the quantity defined before
Theorem~\ref{thm:compiler}.  If every $f_t$ is $\lambda$-strongly convex and
the source regularizer increments are $\alpha_t$, the corresponding source
inequality is
\[
 \E_U\Reg_T^w(u)
 \le \sum_{t=1}^T\frac{\alpha_t-\lambda w_t}{2}
       \E_U\|y_t-u\|^2
 +2G^2\sum_{t=1}^T\eta_{t-1}\Xi_t
 +\frac{10G^2}{\lambda r}\sum_{t=1}^T w_t\delta_t.
\]
Both displays are fixed-horizon statements.  They remain valid after
conditioning on any sigma-field that fixes weights, rates, radii, and delays
while leaving the perturbation directions independent and uniform.
\end{lemma}

\begin{lemma}[Bernoulli hard-capacity wrapper]
\label{lem:source-wrapper}
Suppose a run has backlog at most a public $H$ and independently assigns each
new index the infinite proxy lifetime with probability
\[
 p=\min\{1,C/[8(H+1)]\}.
\]
The source preemptive wrapper maintains $|\cS_t|\le C$ pathwise and produces
the analysis weight
\[
 w_t=\mathbf1\{|\cS_t|<C\}\,I_t/p,
 \qquad I_t\sim\mathrm{Ber}(p).
\]
For every comparator $u$,
\[
 \E\Reg_T(u)
 \le \E\Reg_T^w(u)
 +GD\sum_{t=1}^T\Pr(|\cS_t|=C),
\]
and, under $p(H+1)\le C/8$,
\[
 \Pr(|\cS_t|=C)\le e^{-C}
\]
for every round.  In the strongly-convex interface, the expectation of the
residual term $(\lambda/2)(1-w_t)\|y_t-u\|^2$ is likewise supported on the
saturation event $|\cS_t|=C$.
\end{lemma}

\section{Proof of the conditional-energy compiler}

Condition on the complete scheduler-side trajectory.  The weights, delays,
rates, and smoothing radii are fixed, while the global independence condition
ensures that one-point directions remain independent and uniform.  The
convex part of Lemma~\ref{lem:source-ftbl} gives
\[
 \E_U[\Reg_T^w(u)\mid\mathcal F^{\mathrm{sch}}]\le\frac{D^2}{2\eta_T}
 +2G^2\sum_t\eta_{t-1}\Xi_t
 +\frac{3GD}{r}\sum_t w_t\delta_t.
\]
Taking conditional expectations uses only the stated scheduler filtration.
For $t\ge2$, $h_t\le A_0\le A_{t-1}$ and
\[
 \frac{h_t}{\sqrt{A_{t-1}}}
 \le(1+\sqrt2)(\sqrt{A_t}-\sqrt{A_{t-1}}).
\]
The first term is bounded separately by
$h_1/\sqrt{A_1}\le\sqrt{A_1}$.  The sum telescopes, while
$D^2/(2\eta_T)=(GD/2)\sqrt{A_T}$.  Jensen's inequality gives the second
display in Theorem~\ref{thm:compiler}.

For Lemma~\ref{lem:padding}, fix the oblivious delay vector.  The global
backlog and every public-threshold crossing are then deterministic functions
of time.  Extend the run to horizon $T$ by assigning every post-end round zero
analysis weight and zero loss while keeping the last legal rates.  The padded
game has exactly the original run's regret and energy and satisfies the
fixed-horizon interface in Lemma~\ref{lem:source-ftbl}.

\section{Proof of the proxy-energy lemma}

Every pending nonzero-weight item lies in the tracking set and has weight
$1/p$.  Therefore
\[
 \sum_{s\in\cB_t}w_s=c_t/p,\qquad
 \sum_{s\in\cB_t}w_s^2=c_t/p^2.
\]
Taking expectation over only the current Bernoulli coin proves the exact
identity.  The tracking set is a subset of locally pending indices whose
proxy coin succeeded, so
$\E c_t\le p\sigma_t^E$ and
$\E\sqrt{c_t}\le\sqrt{p\sigma_t^E}$.  Summation and Cauchy--Schwarz give the
area bound.  Finally, $z\sqrt c\le(z^2+c)/2$ and the definition of $p$ give
the uniform initial envelope.

\section{Proof of charged censoring}

For the delay-clock boundary set $X$, decompose epoch regret into rounds in
$X$ and its complement.  The former contribute at most $GD|X|$.  Conditional
on the fixed delay vector, $X$ is fixed independently of admission coins,
directions, and loss values.  On complement rounds, the original Bernoulli
admission law, capacity indicator, importance weight, and saturation event
are unchanged, so Lemma~\ref{lem:source-wrapper}'s per-round conditioning may be
summed over the complement alone.

No item in $X$ produces feedback before the boundary.  Assigning it eventual
analysis weight zero therefore leaves all epoch predictions, tracking
decisions, envelopes, and updates unchanged.  The base terminates at the
boundary and the eventual empty acknowledgment is used only by the external
delay clock.  Lemma~\ref{lem:padding} supplies the fixed-horizon extension.
This proves Lemma~\ref{lem:censor}.

\section{Proof details for the total-delay theorem}

For an epoch $j$ of length $L_j$ and local backlog area $D_j$,
Theorem~\ref{thm:compiler} and Lemma~\ref{lem:proxy} give
\[
 R_j\le cGD\left[\sqrt{D_j}
 +z_j\sqrt{\frac{L_j+1}{p_j}}+\sqrt{H_j+1}
 +L_j\frac{\delta_j}{r}\right]+GDL_je^{-C}.
\]
Round--pending-time pairs belong to only one epoch, so
$\sum_jD_j\le\dtot$.  Since $B_C\ge C/16$ and
$H_j+1=2^jB_C$, the number $K$ of nonempty epochs satisfies $K\le E_C$,
and hence
\[
 \sum_j\sqrt{D_j}\le\sqrt{E_C\dtot}.
\]
The modified initial threshold is important: $p_0$ may equal one, but for
$j\ge1$ the inverse proxy scales grow geometrically.  Consequently
\[
 \sum_jp_j^{-1/2}\le c p_J^{-1/2},
 \qquad
 \sum_jp_j^{-1/4}\sqrt{L_j+1}
 \le c p_J^{-1/4}\sqrt{T+K}.
\]
In the unclipped regime,
$z_j=c^{-1}\sqrt{\nu k}\,T^{1/4}p_j^{1/4}$, so the one-point terms are at
most
\[
 cT^{3/4}\sqrt{\nu k}\,p_J^{-1/4}.
\]
The smoothing terms obey the same bound directly from
$\sum_jL_jp_j^{-1/4}\le Tp_J^{-1/4}$.  Here
$p_J^{-1/4}=O((1+\smax/C)^{1/4})$, so neither contribution has an epoch
logarithm.  Startup terms form a geometric series.  The pending sets censored
at successive boundaries are disjoint and have total size
$O(\smax+1)=O(\sqrt{\dtot}+1)$.  Finally
$T e^{-C}\le e^{-1}$.  This proves Theorem~\ref{thm:convex}; if the smoothing
ratio clips at one, the target already dominates the trivial $GDT$ bound.

\section{Proof details for strong convexity}

The strongly-convex part of Lemma~\ref{lem:source-ftbl} is
\begin{align*}
 \E \Reg_T^w(u)\le{}&
 \E\sum_t\frac{\alpha_t-\lambda w_t}{2}\|y_t-u\|^2
 +2G^2\sum_t\eta_{t-1}\E\Xi_t
 +\frac{10G^2}{\lambda r}\sum_t\E[w_t\delta_t].
\end{align*}
For the shifted schedule, $\alpha_t=\lambda$ for $t\ge2$.  The corresponding
expectation is supported on saturation: with
$V_t=(\lambda/2)\|y_t-u\|^2$ and pre-admission capacity indicator $Q_t$,
\[
 \E[(1-w_t)V_t\mid\mathcal F_t^{\mathrm{sch}}]=(1-Q_t)V_t.
\]
The extra initial term is at most
$\lambda AD^2/2\le2G^2A/\lambda$ because
$D\le2G/\lambda$.

The proxy moments and AM--GM yield
\[
 \E\Xi_t\le\frac{3}{2}\left(\frac{z_t^2}{p}+\sigma_t\right),
 \qquad
 \eta_{t-1}\le\frac{2}{\lambda(A+t)}.
\]
With $x_t=\delta_t/r=[A/(A+t)]^{1/3}$,
\[
 \frac{z_t^2}{p(A+t)}=A^{1/3}(A+t)^{-1/3}.
\]
Both this term and $\sum_t\delta_t/r$ are at most
$\tfrac32A^{1/3}T^{2/3}$, while the remaining delay contribution is exactly
$H_A(d)$.  Saturation contributes $O(G^2/\lambda)$.  This proves
Theorem~\ref{thm:strong}.

\section{Proof of the temporal separation}

For the early vector, the first delayed observation arrives only after the
prediction at round $h+1$.  Therefore every prediction through that round is
independent of the hidden sign and has expected loss at least zero.  Even if
the sign is revealed for free afterwards, every remaining loss is at least
$-1/8$, whereas the static comparator has total loss $-T/8$.  The regret is
at least $(h+1)/8$.

For the late vector, all nonzero backlog occurs after $T-2h$, and its total
area is $h^2$.  Therefore
\[
 H_A(d^{\mathrm{late}})\le\frac{h^2}{A+T-2h}=O(h^2/T).
\]
Theorem~\ref{thm:strong} gives the late upper bound.  The exponent conditions
$\beta>2/3$ and $\beta<5/6$ respectively make the early lower exceed the
one-point term and the late harmonic term no larger than it.

\section{Proof of the capacity-starvation lower bound}
\label{app:capacity-lower}

\begin{lemma}[Adaptive query lower bound]
\label{lem:compact-query}
There are universal constants $c_0,a_0>0$ and a fixed compactly supported
mean-zero noise density $g$ with the following property.  For every $k,m\ge1$
there is a family
\[
 \{F_v:v\in\{-\mu,+\mu\}^k\}
\]
on $B_2^k$, with every $F_v$ being $1/2$-strongly convex, $7/2$-smooth, and
$4$-Lipschitz, such that
every randomized procedure using at most $m$ adaptive queries to
$F_v(x)+\xi$, $\xi\sim g$, has some fixed $v$ for which
\[
 \E[F_v(\widehat x)-F_v(x_v^*)]
 \ge c_0\min\{1,k/\sqrt m\}.                    \tag{20}
\]
The realized noisy functions are uniformly bounded on $B_2^k$.
\end{lemma}

\begin{proof}
For
$v\in\{-\mu,+\mu\}^k$, let
\[
 F_v(x)=\|x\|^2-
 \sum_{i=1}^k\frac{v_ix_i}{1+(x_i/v_i)^2}.
\]
The coordinate calculations in \citet[Theorem~7]{shamir-derivative-free-2013}
show that, for a universal $a_0>0$, $F_v$ is $1/2$-strongly convex,
$7/2$-smooth, and $4$-Lipschitz on $B_2^k$ when
$\sqrt{k}\mu\le1/2$; its minimizer is $x_v^*=a_0v$; and adjacent sign
vectors satisfy
\[
 \sup_x|F_v(x)-F_{v'}(x)|\le\mu^2.               \tag{21}
\]

Let $g=h^2$ be the mean-zero density supported on $[-1,1]$ with
$h(u)=\cos(\pi u/2)\mathbf1\{|u|\le1\}$.  It satisfies
$\|h\|_2=1$ and $\|h'\|_2^2=\pi^2/4$.  If $P_a,P_b$ are translates of this
density, the fundamental theorem of calculus in $L_2$ gives
\[
 H^2(P_a,P_b)=\frac12\|h(\cdot-a)-h(\cdot-b)\|_2^2
 \le\frac{\pi^2}{8}(a-b)^2.                    \tag{22}
\]
For adaptive queries, condition on each common transcript prefix.  The next
query kernel is then the same under adjacent hypotheses, so the sequential
Hellinger tensorization inequality and (21)--(22) give transcript squared
Hellinger distance at most $c m\mu^4$ after $m$ queries.  Assouad's lemma
therefore leaves a constant sign-error probability in every coordinate when
$m\mu^4$ is a sufficiently small constant.  On a sign error in coordinate
$i$, $|\widehat x_i-a_0v_i|\ge a_0\mu$; strong convexity therefore converts
the expected Hamming error to optimization error $ck\mu^2$.  Taking
\[
 \mu=c_1\min\{m^{-1/4},k^{-1/2}\}
\]
proves (20).
The compact additive noise preserves all derivative bounds and makes each
realized loss uniformly bounded.
\end{proof}

\begin{lemma}[Pathwise capacity-to-query reduction]
\label{lem:capacity-query}
Fix $N,d,C\ge1$ and suppose the first $N$ rounds have delay $d$.  For any
legal hard-capacity learner, take the prefix cut immediately after $x_N$ is
generated and before round-$N$ feedback delivery.  Let $Q_N$ be the
first-block indices delivered by that cut.  Then, pathwise,
\[
 |Q_N|\le m_0:=\min\{N,\lfloor CN/d\rfloor\}.   \tag{23}
\]
Moreover, for any stochastic value oracle with independent additive noise,
the joint law of $x_{1:N}$ and all feedback visible by the cut can be
simulated with at most $m_0$ adaptive oracle queries.  In particular the
simulator can output $\bar x=N^{-1}\sum_{t=1}^Nx_t$.
\end{lemma}

\begin{proof}
Let
$Q_N$ contain the first-block indices whose feedback is delivered by the cut
just before round-$N$ feedback delivery, after $x_N$ has been generated.
Every $t\in Q_N$ must have occupied a tracking slot continuously for $d$
rounds; a preempted index cannot be re-added.  Hence, pathwise,
\[
 d|Q_N|\le\sum_{s=1}^N|\cS_s|\le CN,
\]
Counting one endpoint adds one to $d$ and only strengthens the inequality.

Simulate the learner only until it generates $x_N$.  Query a stochastic
value oracle at a first-block point $x_t$ only when that index expires before
the cut.  The unseen value of index $t$ cannot affect its earlier tracking or
preemption decisions, so deferred sampling preserves the prefix transcript
law.  Feedback delivered after the cut cannot affect $x_{1:N}$.  Thus the
simulation is a valid procedure with at most $m_0$ adaptive queries and
output $\bar x=N^{-1}\sum_{t=1}^Nx_t$.
\end{proof}

\begin{proof}[Proof of Theorem~\ref{thm:capacity-lower}]
Fix an arbitrary randomized capacity learner and put $N=\lfloor T/2\rfloor$.
Apply Lemma~\ref{lem:capacity-query} to the stochastic oracle in
Lemma~\ref{lem:compact-query}.  The latter selects a fixed $v$ for the
resulting at-most-$m_0$-query simulator.  By convexity and (20),
\[
 \sum_{t=1}^N\E[F_v(x_t)-F_v(x_v^*)]
 \ge N\E[F_v(\bar x)-F_v(x_v^*)]
 \ge c_0N\min\{1,k/\sqrt{m_0}\}.               \tag{24}
\]

Finally draw $\xi_1,\ldots,\xi_T$ independently from $g$ at time zero and set
$f_t(x)=F_v(x)+\xi_t$ for every $t$.  This is an oblivious distribution over
loss sequences in the claimed function class, all with common minimizer
$x_v^*$.  For every realization,
\[
 \Reg_T=\sum_{t=1}^T[F_v(x_t)-F_v(x_v^*)]
 \ge\sum_{t=1}^N[F_v(x_t)-F_v(x_v^*)].
\]
Combining with (24), $N=\Theta(T)$, and
$m_0\le\min\{N,CN/d\}$ proves (19).  Averaging over the time-zero noise draw
selects one fixed deterministic realization with at least the same expected
regret over the learner's randomness, completing the minimax quantifiers.
\end{proof}

\end{document}